\documentclass{article}

 \usepackage[preprint]{neurips_2025}

\usepackage[utf8]{inputenc} 
\usepackage[T1]{fontenc}    
\usepackage{hyperref}       
\usepackage{url}            
\usepackage{booktabs}       
\usepackage{amsfonts}       
\usepackage{nicefrac}       
\usepackage{microtype}      
\usepackage{xcolor}         
\usepackage{amsmath}
\usepackage{amsthm}

\usepackage{graphicx}
\usepackage{wrapfig}

\newcommand{\ucb}{\text{UCB}}

\usepackage{algorithmic}
\usepackage{algorithm}

\usepackage{cleveref}
\newtheorem{lemma}{Lemma}
\newtheorem{assume}{Assumption}

\newtheorem{prop}{Proposition}
\newtheorem{defn}{Definition}

\newcommand{\dom}{\mathbb{D}}
\newcommand{\doint}{\operatorname{do}}
\newcommand{\parents}{\operatorname{PA}}
\newcommand{\abs}{\boldsymbol{\alpha}}

\title{Using causal abstractions to accelerate decision-making in complex bandit problems}

\author{%
  Joel Dyer\thanks{Equal contribution. Correspondence to: \texttt{joel.dyer@cs.ox.ac.uk};$\ \ $   \texttt{nicholas.bishop@cs.ox.ac.uk};$\ \ $ \texttt{fabio.zennaro@uib.no}.}\\
  University of Oxford
  \And
  Nicholas Bishop\textsuperscript{$\ast$}\\
  University of Oxford
  \AND
  Anisoara Calinescu\\
  University of Oxford
  \And
  Michael Wooldridge\\
  University of Oxford
  \And
  Fabio Massimo Zennaro\\
  University of Bergen
}

\begin{document}

\maketitle

\begin{abstract}

    Although real-world decision-making problems can often be encoded as causal multi-armed bandits (CMABs) at different levels of abstraction, a general methodology exploiting the information and computational advantages of each abstraction level is missing. In this paper, we propose AT-UCB, an algorithm which efficiently exploits shared information between CMAB problem instances defined at different levels of abstraction. More specifically, AT-UCB leverages causal abstraction (CA) theory to explore within a cheap-to-simulate and coarse-grained CMAB instance, before employing the traditional upper confidence bound (UCB) algorithm on a restricted set of potentially optimal actions in the CMAB of interest, leading to significant reductions in cumulative regret when compared to the classical UCB algorithm. We illustrate the advantages of AT-UCB theoretically, through a novel upper bound on the cumulative regret, and empirically, by applying AT-UCB to epidemiological simulators with varying resolution and computational cost.
  
\end{abstract}

\section{Introduction}
\label{sec:intro}
Complex decision-making problems are naturally represented and studied as causal systems at multiple levels of abstraction. Consider the problem of controlling the spread of a disease through a population. One may develop a fine-grained, computationally-expensive agent-based model (ABM) as well as  a low-fidelity differential equation model (DEM) that is cheap to simulate. By experimenting with the DEM, we can quickly identify policy interventions that might be worth testing in the higher resolution ABM. But how should we divide our computational budget between the two models, and how should we use information learned from the DEM to inform our search over interventions in the ABM? A similar challenge appears in fields such as economics and ecology, that often require intervening in complex environments. More generally, integrating knowledge across different models is challenging, and practitioners often rely on ad-hoc solutions.

In this paper, we study how multiple causal models at different levels of abstraction can be formally related and jointly exploited to improve decision-making. We adopt the framework of causally abstracted multi-armed bandits (CAMABs) \citep{zennaro2024causally}, where causal systems are encoded as structural causal models (SCMs), relations across SCMs are formalized via causal abstraction (CA), and decision problems are expressed as multi-armed bandits (MABs). We focus on the scenario where a decision-maker is given two models, a realistic but expensive base model and a simplified but cheap abstracted model; their aim is to take advantage of running the cheap model in order to gain information for discovering optimal policies in the realistic model more efficiently. We propose abstract thresholding UCB (AT-UCB), an algorithm that first performs exploration in the abstracted model in order to filter out sub-optimal interventions before running the UCB algorithm on the base model over a restricted set of actions. We provide: (i) a theoretical analysis of the regret of AT-UCB, and (ii) simulations confirming our theoretical results.

\textbf{Related Work} Our work adopts the CAMAB framework introduced by \cite{zennaro2024causally}, which also introduces methods for transferring information across causal multi-armed bandit problems (CMABs). However, whilst their work focuses on transferring information from the base model to the abstracted model, we consider the more realistic inverse setting and provide a more rigorous theoretical analysis. More generally, CAMABs are subtly related to MAB problems which assume a form of clustering between arms such as regional bandits \citep{wang2018regional}, MABs with dependent arms \citep{singh2024multi}, and bandits with similar arms \citep{pesquerel2021stochastic}. Our AT-UCB algorithm is reminiscent of approaches that warm-start bandit algorithms with offline or confounded data \citep{cheung2024leveraging,yang2025best, sharma2020warm} as well as thresholding bandits \citep{mason2020finding, locatelli2016optimal}. For more discussion of related MAB problems, 
we refer the reader to \Cref{app:literatute_review}.

\section{Background}

Before describing AT-UCB in detail, we provide some background on CAMABs. We first define  structural causal models (SCMs) and interventions \citep{pearl2009causality,peters2017elements}.

\begin{defn}[SCM]
A structural causal model (SCM) $\mathcal{M}$ is a tuple $\langle \mathcal{X}, \mathcal{U}, \mathcal{F}, P \rangle$ where $\mathcal{X}=\left\{ X_i \right\}_{i=1}^d$ is a collection of $d$ endogenous variables, $\mathcal{U}=\left\{ U_i \right\}_{i=1}^d$ is a collection of exogenous variables, $\mathcal{F}=\left\{ f_{X_i} \right\}_{i=1}^d$ is a collection of structural functions $f_{X_i}: \dom[{\parents}(X_i)] \times \mathbb{U}_i \rightarrow \mathbb{X}_i$, where $\dom$ denotes the domain and with ${\parents}(X_i)\subseteq \mathcal{X} \setminus X_i$, and $P$ is a probability distribution over $\mathcal{U}$. 
\end{defn}

\begin{defn}[Intervention]
Given an SCM $\mathcal{M}$, an intervention $a={\doint}(\mathbf{X}=\mathbf{x})$ on a set of endogenous variables $\mathbf{X}$ associated with values $\mathbf{x}$, defines a new SCM $\mathcal{M}^{a}$ by replacing $f_i$ with the constant $x_i$ for each variable $X_i \in \mathbf{X}$.
\end{defn}

We assume that our SCMs induce directed acyclic graphs (DAGs) over the endogenous variables $\mathcal{X}$, such that $\parents$ is interpreted as graph-theoretical parents. For a given intervention $a$, each variable $X$ has an interventional distribution defined by pushing $P$ through the modified structural equations of the model $\mathcal{M}^{a}$. For ease of presentation, we let $X_{a} \sim P(X \mid a) $ denote a random variable distributed according to $X$ under the intervened model $\mathcal{M}^{a}$. An SCM naturally induces a multi-armed bandit problem known as a CMAB \citep{bareinboim2015bandits,lattimore2016causal}.
\begin{defn}[CMAB]
A causal multi-armed bandit (CMAB) $\mathcal{B}$ is a tuple $\langle \mathcal{M}, Y, \mathcal{A} \rangle$ where $\mathcal{M}$ is an SCM, $Y \subseteq \mathbb{R}$ is an endogenous outcome/reward variable in $\mathcal{X}$, and $\mathcal{A}=\{a_i\}_{i=1}^k$ is a set of $k$ interventions/arms/actions on $\mathcal{M}$.
\end{defn}
 A CMAB describes a basic sequential decision problem over a time 
horizon of length $n \in \mathbb{N}_{+}$. On each time step $t \in [n]$ the learner (who has no knowledge of $P$ or $\mathcal{F}$) is tasked with selecting an intervention $a \in \mathcal{A}$. Upon selecting intervention $a^{(t)}$ on time step $t$, the learner receives a reward $y^{(t)} \sim Y_{a^{(t)}}$. The goal of the learner is to maximize their cumulative reward over the time horizon. We let $\mu_{a^{(t)}}$ describe the mean reward associated with action $a$ and (for the sake of analysis) label actions so that $\mu_{a_{1}} \geq \mu_{a_{2}} \geq \cdots \geq \mu_{a_{k}}$. 

Lastly, we define $(\tau,\omega)$-abstraction, which formally relates two SCMs through a state map $\tau$ and an intervention map $\omega$. We also define two measures that capture the faithfulness of a $(\tau, \omega)$-abstraction in the bandit context; the interventional consistency (IC) error \citep{rubenstein2017causal,beckers2018abstracting, zennaro2023quantifying} and the reward discrepancy (RD) \citep{zennaro2024causally}.

\begin{defn}[Abstraction]
    Given a base SCM $\mathcal{M}$ with a set of interventions $\mathcal{A}$ and an abstracted SCM $\mathcal{M'}$ with a set of interventions $\mathcal{A}'$, an abstraction $\abs$ is a tuple $(\tau,\omega)$ where $\tau: \dom[\mathcal{X}]\rightarrow\dom[\mathcal{X'}]$ maps outcomes of the base model to outcomes of the abstracted model, and $\omega: \mathcal{A}\rightarrow\mathcal{A}'$ maps base interventions to abstracted interventions. 
\end{defn}


\begin{defn}[Abstraction measures (IC,RD)]
   Given an abstraction $\abs$, the interventional consistency (IC) error is:
\begin{equation}
    e(\abs) := \max_{a \in \mathcal{A}} D_{W_2} \left( \tau(P(Y\vert a)), P(Y'\vert \omega(a))) \right)
\end{equation}
and the reward discrepancy (RD) error is:
\begin{equation}
    s(\abs) := \max_{a \in \mathcal{A}} D_{W_2} \left( P(Y\vert a), \tau(P(Y\vert a)) \right),
\end{equation}
where $D_{W_2}$ is the $2$-Wasserstein distance, and $\tau(P(Y|a))$ is the distribution of $\tau(Y)$ when $Y \sim P(\cdot \mid a)$.

\end{defn}

In short, the IC error measures the difference between (i) intervening in the base model and abstracting the result, and (ii) abstracting the intervention and intervening in the abstracted model. Meanwhile, RD measures the degree to which rewards are distorted by the state map $\tau$. A CAMAB \citep{zennaro2024causally} consists of two CMABs and an abstraction formally relating them.

\begin{defn}[CAMAB]
    A causally abstracted MAB (CAMAB) $\mathcal{C}$ is a tuple $\langle \mathcal{B}, \mathcal{B'}, \abs \rangle$ made up by a base CMAB $\mathcal{B}$ and an abstracted CMAB $\mathcal{B}'$ related by an abstraction $\abs$.
\end{defn}
To distinguish mathematical quantities relating to $\mathcal{B}'$ rather than $\mathcal{B}$ we append an apostrophe. For example, the total number of arms in $\mathcal{B}$ is denoted by $k$, whilst the total number of arms in $\mathcal{B}'$ is denoted by $k'$. Given a CAMAB $\mathcal{C} = \langle \mathcal{B}, \mathcal{B'}, \abs \rangle$, we consider the following scenario. Assume that the learner will interact with the base CMAB $\mathcal{B}$ over a time horizon of length $n$. Prior to the start of the time horizon the learner is permitted to interact with the abstracted CMAB $\mathcal{B}^{\prime}$ for as long as they wish. However, each time the learner intervenes in $\mathcal{B}^{\prime}$ they pay a fixed cost $C > 0$. Note that this setting resembles the epidemiological example given in Section \ref{sec:intro}. 
A natural goal for the learner is to minimize their regret with respect to an oracle who does not spend any time interacting with $\mathcal{B}^{\prime}$ and takes the optimal intervention $a_{1}$ on every time step when interacting with $\mathcal{B}$.





\begin{defn}[Cumulative Regret]
    Given a CAMAB $\mathcal{C}$, the expected cumulative regret of the learner is $\mathcal{R}_{n}(\mathcal{C}) := n'\cdot C + \sum_{t=1}^n \Delta_{a^{(t)}}$, where $n^{\prime}$ denotes the number of interventions performed in $\mathcal{B}^{\prime}$, and $\Delta_{a} = \mu_{a_1} - \mu_{a}$ denotes the optimality gap between action $a$ and action $a_{1}$.
\end{defn}

\section{Methods and theory}

To propose an algorithm that can exploit the structure of a CAMAB, we first characterize the relation between: (i) the mean $\mu_{a'_1}$ of the true optimal arm in the abstracted CMAB $\mathcal{B}'$; and (ii) the mean $\mu_{\omega(a_1)}$ of the image through $\omega$ of the true optimal arm in the base CMAB $\mathcal{B}$.

\begin{prop}\label{prop:bound_absoptbase}
Given a CAMAB $\mathcal{C}$, let $\abs = \langle \tau, \omega \rangle$ be an approximate abstraction with IC error $e(\abs)$ and RD error $s(\abs)$, and define $\epsilon(\abs) = 2\left(s(\abs)+e(\abs)\right)$. Then it holds that:
    \begin{equation}
        \mu_{a'_1} - \epsilon(\abs) \leq \mu_{\omega(a_1)} \leq \mu_{a'_1}.
    \end{equation}
\end{prop}

Informally, \Cref{prop:bound_absoptbase} bounds the distance between the abstracted optimal arm and the image of the base optimal arm. It thus suggests that the search for the optimal base arm $a_1$ can be accelerated in the presence of an approximate causal abstraction $\abs$ by restricting attention to base interventions in an interval defined by \Cref{prop:bound_absoptbase} as $a \in \omega^{-1}\left(\mathcal{D}^c\right)$, where $\mathcal{D}^c$ denotes the complement of the set:
\begin{equation}
    \mathcal{D} := \left\{ a' \in \mathcal{A}':\mu_{a'}<\mu_{a_{1}'}-\epsilon(\abs) \right\}.
\end{equation}
This observation informs our contribution, which is outlined in \Cref{alg:at-ucb} and combines the initial use of a thresholding bandit algorithm in the abstract CMAB -- performed through uniform exploration of the $k'$ abstract arms, for the sake of simplicity -- with a standard upper confidence bound (UCB) \citep{lattimore2020bandit} algorithm in the base CMAB. 

\begin{algorithm}[tb]
	\caption{AT-UCB}\label{alg:at-ucb}
	\begin{algorithmic}[1]
		\STATE { \textbf{Input:} } Base and abstract CMABs $\mathcal{M}$, $\mathcal{M}'$; base and abstract time horizons $n$, $n'$; UCB parameter $\delta > 0$; approximate abstraction $\abs$; threshold $\varepsilon > 0$
		\STATE { \textbf{Output:} } estimated optimal arm $\hat{a}_1$
		
		\STATE Pull each abstract arm $a' \in \mathcal{A}'$ a total of $n'/k'$ times each to construct estimates $\hat{\mu}_{a'}$.
            \STATE Construct 
            \begin{equation}
                \hat{\mathcal{D}} = \{ a' \in \mathcal{A'} \mid \hat{\mu}_{a'} < \max_{\tilde{a}' \in \mathcal{A}'} \hat{\mu}_{\tilde{a}'} - 
                \varepsilon
                \}.
            \end{equation}
		\STATE Run UCB on $\hat{\mathcal{A}} := \omega^{-1}(\hat{\mathcal{D}}^c) = \cup_{a' \in \hat{\mathcal{D}}^c} \omega^{-1}(a')$ for $n$ time steps and with parameter $\delta$.
		
		\STATE { \textbf{Return:} } $\hat{a}_1 = \arg\max_{a \in \hat{\mathcal{A}}} \ucb_{a}\left(n, \delta\right)$
		
	\end{algorithmic}
\end{algorithm}

\subsection{Regret analysis}\label{sec:regret}

Next, we analyse the expected cumulative regret incurred by \Cref{alg:at-ucb}. We work under the standard assumption of 1-subgaussianity on the reward distributions and assume the abstract horizon $n'$ is sufficiently large to bound exponentials appearing as a result of the application of Chernoff bounds when bounding the regret incurred by experimenting in $\mathcal{M}$: 

%
%
%


\begin{assume}\label{ass:abs_hor_len}
    With $\Delta'_{a'} := \mu_{a'_1} - \mu_{a'}$ , the abstract horizon length $n'$ satisfies,
    \begin{equation}
        n' \geq 4k' \log(nk') \max\left\{{\left(\epsilon(\abs) - \Delta'_{\omega(a_1)}\right)^{-2}} , {\left(\min_{a'\in\mathcal{D}}\Delta'_{a'} - \epsilon(\abs)\right)^{-2}} \right\}.
    \end{equation}
\end{assume}


Under these assumptions, the following upper bound on the expected cumulative regret for \Cref{alg:at-ucb} holds, with the proof deferred to \Cref{app:proofs}:

\begin{prop}\label{prop:at-ucb_ecr}
    The cumulative regret $\mathcal{R}_{n}(\mathcal{C})$ under \Cref{alg:at-ucb} with $\varepsilon = \epsilon(\abs)$ satisfies
    \begin{equation}
    \mathcal{R}_{n}(\mathcal{C}) \leq n' \cdot C + \left(3\sum_{a \in \omega^{-1}(\mathcal{D}^c)} \Delta_{a} + \sum_{\substack{a \in \omega^{-1}(\mathcal{D}^c),\\\Delta_a > 0}} \frac{16 \log(n)}{\Delta_a}\right) + 2 \max_{a\in \mathcal{A}} \Delta_a.
\end{equation}
\end{prop}
When $C$ is sufficiently small and $\epsilon(\abs)$ is small enough that $\vert{\omega^{-1}(\mathcal{D}^c)\vert} \ll k$, \Cref{prop:at-ucb_ecr} gives a lower regret bound than the standard upper bound for UCB.

\section{Experiments}\label{sec:exp}

In this section, we study \Cref{alg:at-ucb} empirically, using two epidemic simulators. For $\mathcal{M}$, we consider a susceptible-infected-recovered-susceptible model. In this model, $k$ spatial regions are connected in a fully-connected graph, and different regions carry and spread the hypothetical virus more easily. The epidemic is simulated over discrete time steps of size $\Delta t > 0$. Further details are provided in \Cref{app:sirs_comms}. 
The learner is tasked with choosing a community $a \in \{1,\ldots,k\}$, on which  a fixed-duration lockdown, starting at a prespecified time, will be imposed. The reward observed is the negation of total infections over the simulated horizon. The learner enjoys access to a high-level model $\mathcal{M}'$, comprised of large abstract communities, related to $\mathcal{M}$ via an approximate causal abstraction. Abstract communities $a' \in \{1,\ldots,k'\}$ with $k' < k$ are related to base communities via $\omega$, and $\mathcal{M}'$ is simulated at a coarser time resolution than $\mathcal{M}$, such that $\Delta t' > \Delta t$. Further details on the abstraction used are provided in \Cref{app:sirs_comms_abs}.

\begin{wrapfigure}{r}{0.3\textwidth}
  \vspace{-1.7em}
  \begin{center}
    \includegraphics[width=0.3\textwidth]{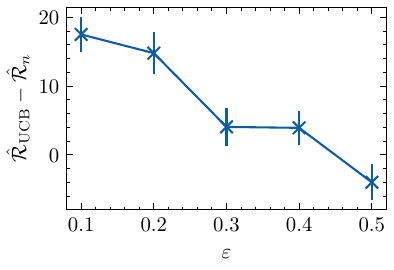}
  \end{center}
  \vspace{-1em}
  \caption{Difference in empirical average cumulative regret as a function of $\varepsilon$ for the experiment in \Cref{sec:exp}.}\label{fig:sirs_nprime20}
  \vspace{-2.4em}
\end{wrapfigure}

In \Cref{fig:sirs_nprime20}, we show the difference in cumulative regret from running UCB in the base bandit and running AT-UCB with $n'=20$, as a function of $\varepsilon$. Both procedures use $n=100$, and error bars show the standard error in the mean from 10 repeats. Larger values indicate that AT-UCB achieves comparatively lower regret. For sufficiently small $\varepsilon$, a lower regret is achieved using the approximate causal abstraction and \Cref{alg:at-ucb} instead of the standard UCB algorithm. In this instance, since $\omega(a_1) = a_1'$ and the abstract reward distributions have low variance relative to the distances between arm means, the difference in regret decreases with increasing $\varepsilon$ as suboptimal arms are included while $\omega(a_1)$ is rarely excluded even for low $\varepsilon$.

\section{Conclusion}
\label{sec:con}
We have introduced AT-UCB, a CAMAB algorithm that exploits CA theory to explore an abstract CMAB so that a more complex CMAB can be solved efficiently via immediate elimination of sub-optimal arms. Whilst we have shown that AT-UCB provides tighter regret guarantees given a sufficiently accurate abstraction, AT-UCB has several limitations, discussed further in \Cref{app:lims}, that we aim to address in future work. In particular, AT-UCB does not exploit causal information provided by the SCMs $\mathcal{M}$ and $\mathcal{M}'$ to expedite exploration. Similarly, assessing the validity of \Cref{ass:abs_hor_len} a priori relies on knowledge of abstract optimality gap sizes, which may be unrealistic. Moreover, the regret guarantee provided by \Cref{prop:at-ucb_ecr} relies on setting $\varepsilon = \epsilon(\abs)$, which may be unknown.

\ack{The authors are extremely grateful to Theodoros Damoulas for his feedback, helpful discussions, and his support. JD, NB, AC, and MW acknowledge funding from a
UKRI AI World Leading Researcher Fellowship awarded
to Wooldridge (grant EP/W002949/1). MW and AC also
acknowledge funding from Trustworthy AI - Integrating
Learning, Optimisation and Reasoning (TAILOR), a project
funded by the European Union Horizon2020 research and
innovation program under Grant Agreement 952215. }

\bibliographystyle{plainnat}
\bibliography{ref}

\appendix

\section{Further literature review} \label{app:literatute_review}

In this appendix, we discuss further relations between our work and the literature.

\paragraph{CA.}
Differently from \cite{zennaro2024causally} which relies on $\alpha$-abstractions \citep{rischel2020category}, we adopt the formalism of $\tau$-$\omega$ abstraction \citep{rubenstein2017causal,beckers2018abstracting}. Although the two frameworks are related \citep{schooltink2024aligning,beckers2018abstracting}, the $\tau$-$\omega$ abstraction allows us to deal with the intervention map $\omega$ separately from the variable map $\tau$. The $\tau$-$\omega$ formalism was also used to motivate a procedure for learning interventionally consistent surrogate models of complex simulators in \citet{dyer2024interventionally}; the current work differs in that we consider how to exploit knowledge of an approximate causal abstraction to accelerate decision-making, rather than the problem of learning approximate causal abstractions in the first place.

\paragraph{MABs.}
Like regional MABs \citep{wang2018regional} and MABs with dependent arms \citep{singh2024multi}, CAMABs implicitly cluster base actions through abstraction; however, in general, CA cannot provide guarantees on the functional form or the parameters of the actions grouped together. 
As actions are formally related via abstractions in CAMABs, taking a single action (base or abstract) provides information about several other actions in a fashion similar to MABs with side information \citep{mannor2011bandits,qi2024graph, caron2012leveraging}. However, exploiting this information is less trivial in CAMABs, as abstractions may introduce arbitrary biases, as captured through abstraction measures such as the interventional consistency error and reward discrepancy error. AT-UCB uses abstract data to rapidly eliminate arms before interacting with a base model, in a fashion reminiscent of MABs with offline information \citep{cheung2024leveraging,yang2025best,sharma2020warm}, which use offline data to warm start MAB algorithms. Employing such warm start algorithms in the CAMAB setting is complicated by the biases introduced by abstractions discussed previously. The arm elimination procedure adopted by AT-UCB is also similar in nature to thresholding bandit algorithms \citep{locatelli2016optimal}, which attempt to identify a subset of arms with mean reward above given a threshold. Unfortunately, many thresholding bandit algorithms assume a fixed threshold rather than a relative one that depends of the mean reward of the optimal arm, making them incompatible AT-UCB. To the best of our knowledge, the only exception to this is \cite{mason2020finding}, who aim to identify a set of $\epsilon$-satisficing arms.


\paragraph{CMABs.} A wide range of algorithms have been proposed for the CMAB setting which exploit knowledge of the causal graph associated with the underlying SCM to expedite exploration \citep{lattimore2016causal,lu2020regret,bilodeau2022adaptively}. Similar to \citep{lee2019structural}, we also perform arm elimination to efficiently solve a CMAB; however, where \cite{lee2019structural} filters out actions through the identification of a possibly optimal minimal intervention set via causal analysis, we rely on CA to estimate a similar set of possibly optimal actions.
Finally, CAs allow for information transfer akin to transfer learning in CMABs \citep{zhang2017transfer}; however, where transfer learning in  \cite{zhang2017transfer} requires knowing the causal graphs of both SCMs (and their overlap), a $\tau$-$\omega$ abstraction requires only knowledge of a map between variables and between interventions.

\section{Further mathematical results and proofs}\label{app:proofs}

In this appendix, we provide further theoretical results to prove the claims in the main body. For completeness, we recall the following result from \citet{zennaro2024causally}, which we utilize in our proofs.

\begin{lemma}[Proposition 4.1 of \citet{zennaro2024causally}]\label{lem:means_bounded_absrew}
    Given a CAMAB $\mathcal{C}$, for any $a \in \mathcal{A}$ and approximate abstraction $\abs = \langle \tau,\omega \rangle$,
    \begin{equation}
        \vert{\mu_{a} - \mu_{\omega(a)}\vert} \leq s(\abs) + e(\abs).
    \end{equation}
\end{lemma}

\Cref{lem:means_bounded_absrew} bounds the absolute difference between the expected reward associated with intervening with an action $a \in \mathcal{A}$ in the base bandit $\mathcal{B}$ and the expected reward associated with intervening with the corresponding abstracted action $\omega(a)$ in  $\mathcal{B}'$. That is, \Cref{lem:means_bounded_absrew} bounds how much the expected reward associated with a base action changes after abstracting.
The next lemma bounds the optimality gap of an action in $\mathcal{B}$ in terms of the IC and RD error of $\abs$ and the optimality gap of the corresponding action in $\mathcal{B}'$.
\begin{lemma}\label{lem:opt_gaps_base_abs}
    Given a CAMAB $\mathcal{C}$, for any $a_k \in \mathcal{A}$ and approximate abstraction $\abs  = \langle \tau,\omega \rangle$, 
    \begin{equation}
        \Delta_{a_k} \leq \Delta'_{\omega(a_k)} + 2(s(\abs)+e(\abs)) - \Delta_{\omega(a_1)}'.
    \end{equation}
\end{lemma}

\begin{proof}[Proof of \Cref{lem:opt_gaps_base_abs}]
For any $a_k \in \mathcal{A}$,
    \begin{flalign*}
        &&\Delta_{a_k}&= \mu_{a_1} - \mu_{a_k}&&\\
        && &=    \mu_{a_1} + \mu_{a_1'} - \mu_{a_1'} - \mu_{\omega(a_1)} + \mu_{\omega(a_1)}  - \mu_{a_k} + \mu_{\omega(a_k)} - \mu_{\omega(a_k)} &&\\
        && &\leq    \vert{\mu_{a_1} - \mu_{\omega(a_1)}  - \mu_{a_k} + \mu_{\omega(a_k)}\vert} + \mu_{a_1'} - \mu_{a_1'} + \mu_{\omega(a_1)} - \mu_{\omega(a_k)} &&\\
        && &=    \vert{\mu_{a_1} - \mu_{\omega(a_1)}  - \mu_{a_k} + \mu_{\omega(a_k)}\vert} + \Delta'_{\omega(a_k)} - \Delta'_{\omega(a_1)}&&\\
        && &\leq \vert{\mu_{a_1} - \mu_{\omega(a_1)}}\vert + \vert{\mu_{a_k} - \mu_{\omega(a_k)}}\vert + \Delta'_{\omega(a_k)} - \Delta'_{\omega(a_1)}&&\text{(triangle inequality)}\\
        && &\leq 2(s(\abs)+e(\abs)) + \Delta'_{\omega(a_k)} - \Delta'_{\omega(a_1)}. &&\text{(\Cref{lem:means_bounded_absrew})}
    \end{flalign*}
\end{proof}

We now prove the bound on the mean of the abstract arm corresponding to the optimal arm in the base bandit instance, provided in \Cref{prop:bound_absoptbase}:

\setcounter{prop}{0}
\begin{prop}
Given a CAMAB $\mathcal{C}$, let $\abs = \langle \tau, \omega \rangle$ be an approximate abstraction with IC error $s(\abs)$ and RD error $e(\abs)$, and define $\epsilon(\abs) = 2\left(s(\abs)+e(\abs)\right)$. Then it holds that:
    \begin{equation}
        \mu_{a'_1} - \epsilon(\abs) \leq \mu_{\omega(a_1)} \leq \mu_{a'_1}.
    \end{equation}
\end{prop}

\begin{proof}[Proof of \Cref{prop:bound_absoptbase}]
    To prove the right side of the inequality, $\mu_{\omega(a_1)} \leq \mu_{a'_1}$, it is sufficient to recall the definition of the optimal abstract arm $a_1'$, according to which $\mu_{a_1'} \geq \mu_{a_k}$ for all $a_k \in \mathcal{A}'$, including $\omega(a_1)$.

    To prove the left side of the inequality, $\mu_{a'_1} - \epsilon(\abs) \leq \mu_{\omega(a_1)}$, we apply \Cref{lem:opt_gaps_base_abs} with $a_k \in \omega^{-1}(a_1')$
    \begin{flalign*}
        &&  \Delta_{a_k}&\leq \Delta'_{\omega({a_k})} + 2(s(\abs)+e(\abs)) - \Delta_{\omega(a_1)}'&&\text{(\Cref{lem:opt_gaps_base_abs})}\\
        &&  &\leq \Delta'_{a'_1} + 2(s(\abs)+e(\abs)) - \Delta_{\omega(a_1)}'&&\text{(mapping $a_k$ to the optimal $a'_1$)}\\
        && &\leq 2(s(\abs)+e(\abs)) - \Delta_{\omega(a_1)}'&&\text{(optimality gap for $a'_1$ is zero)}\\
        && & =  2(s(\abs)+e(\abs)) -\mu_{a_1'} + \mu_{\omega(a_1)}. &&\text{(definition of optimality gap)}
    \end{flalign*}   
    Rearranging gives
    \begin{flalign*}
        &&\mu_{\omega(a_1)}&\geq \mu_{a_1'} - 2(s(\abs)+e(\abs)) + \Delta_{a_k}&&\\
        && &\geq \mu_{a_1'} - 2(s(\abs)+e(\abs)), &&\text{($\Delta_{a_k} \geq 0$)}\\
        && &\geq \mu_{a_1'} - \epsilon(\abs), &&\text{(definition of $\epsilon(\abs)$)}
    \end{flalign*}
    which proves the left side of the inequality.
\end{proof}

Next, we state and prove an auxiliary lemma, used in the proof  \Cref{prop:at-ucb_ecr}. In short, \Cref{lem:bound_sum_exps} shows that \Cref{ass:abs_hor_len} ensures  AT-UCB obtains enough abstract samples to correctly filter abstract actions with  probability sublinear in $n$.
\begin{lemma}\label{lem:bound_sum_exps}
    Under \Cref{ass:abs_hor_len},
    \begin{equation}
        \label{eq:lem3statement}
        \sum_{a' \in \mathcal{A}'} \exp\left(-\frac{n'}{4k'}\left(\epsilon(\abs) + \Delta_{a'}' - \Delta'_{\omega(a_1)}\right)^2\right) + \sum_{a' \in \mathcal{D}} \exp\left(-\frac{n'}{4k'}\left(\Delta'_{a'} - \epsilon(\abs)\right)^2\right)
        \leq \frac{2}{n}.
    \end{equation}
\end{lemma}
\begin{proof}
    By \Cref{ass:abs_hor_len},
    \begin{flalign}
        \frac{n'}{4k'} \min\left\{ \left(\epsilon(\abs) - \Delta'_{\omega(a_1)}\right)^2, \left(\min_{a' \in \mathcal{D}} \Delta'_{a'} - \epsilon(\abs)\right)^2 \right\} &\geq \log(nk').
    \end{flalign}
    Taking the exponent of both sides yields,
    \begin{flalign}
        \exp \left( \min\left\{ \frac{n'}{4k'} \left[\epsilon(\abs) - \Delta'_{\omega(a_1)}\right]^2, \frac{n'}{4k'} \left[\min_{a' \in \mathcal{D}} \Delta'_{a'} - \epsilon(\abs)\right]^2 \right\} \right) &\geq nk'.
    \end{flalign}
Reciprocating both sides and reversing the inequality yields,       
    \begin{align}
        \frac{1}{n}&  \geq k' \max\left\{ \exp\left(-\frac{n'}{4k'}\left[\epsilon(\abs) - \Delta'_{\omega(a_1)}\right]^2\right), \exp\left(-\frac{n'}{4k'}\left[\min_{a' \in \mathcal{D}} \Delta'_{a'} - \epsilon(\abs)\right]^2\right) \right\}.
    \end{align}
    Recalling that $k' := |\mathcal{A}'|$ gives,
    \begin{align}
        \label{eq:lem3main}
        \frac{1}{n}&  \geq \sum_{a' \in \mathcal{A'}} \max\left\{ \exp\left(-\frac{n'}{4k'}\left[\epsilon(\abs) - \Delta'_{\omega(a_1)}\right]^2\right), \exp\left(-\frac{n'}{4k'}\left[\min_{a' \in \mathcal{D}} \Delta'_{a'} - \epsilon(\abs)\right]^2\right) \right\}.
    \end{align}
    Using \Cref{eq:lem3main}, we may bound both terms on the LHS of \Cref{eq:lem3statement}. Starting with the first term, since $\Delta_{a'}' \geq 0$ and $\epsilon(\abs) - \Delta'_{\omega(a_1)} \geq 0$ (by \Cref{prop:bound_absoptbase}), it follows that,
    \begin{equation}
        \label{eq:lem3p1}
        \exp\left(-\frac{n'}{4k'}\left[\epsilon(\abs) - \Delta'_{\omega(a_1)}\right]^2\right) \geq \exp\left(-\frac{n'}{4k'}\left[\epsilon(\abs) + \Delta'_{a'} - \Delta'_{\omega(a_1)}\right]^2\right).
    \end{equation}
    Combining \Cref{eq:lem3main} with \Cref{eq:lem3p1} yields,
    \begin{equation}\label{eq:1on-a'}
        \frac{1}{n} \geq \sum_{a'\in\mathcal{A}'} \exp\left(-\frac{n'}{4k'}\left[\epsilon(\abs) + \Delta'_{a'} - \Delta'_{\omega(a_1)}\right]^2\right).
    \end{equation}
    Moving onto the second term, since $\min_{a' \in \mathcal{D}} \Delta'_{a'} - \epsilon(\abs) > 0$ by definition of $\mathcal{D}$, it follows that,
    \begin{equation}
        \label{eq:lem3p2}
        \exp\left(-\frac{n'}{4k'}\left[\min_{a' \in \mathcal{D}} \Delta'_{a'} - \epsilon(\abs)\right]^2\right) \geq \exp\left(-\frac{n'}{4k'}\left[\Delta'_{a'} - \epsilon(\abs)\right]^2\right)\qquad \forall a' \in \mathcal{D}.
    \end{equation}
    Combining \Cref{eq:lem3main} with \Cref{eq:lem3p2} yields,
    \begin{equation}\label{eq:1on-d}
        \frac{1}{n} \geq \sum_{a'\in\mathcal{A}'} \exp\left(-\frac{n'}{4k'}\left[\min_{\tilde{a}' \in \mathcal{D}}\Delta'_{\tilde{a}'} - \epsilon(\abs)\right]^2\right) \geq \sum_{a'\in\mathcal{D}} \exp\left(-\frac{n'}{4k'}\left[\Delta'_{a'} - \epsilon(\abs)\right]^2\right).
    \end{equation}
    Finally, adding \Cref{eq:1on-a'} and \Cref{eq:1on-d} gives the result.
\end{proof}

We now prove our main result, \Cref{prop:at-ucb_ecr}:
\begin{prop}\label{prop:at-ucb_ecr}
    The expected cumulative regret $\mathcal{R}_{n',n}(\abs)$ for \Cref{alg:at-ucb} has the upper bound
    \begin{equation}
    \mathcal{R}_{n}(\mathcal{C}) \leq n' \cdot C + \left(3\sum_{a \in \omega^{-1}(\mathcal{D}^c)} \Delta_{a} + \sum_{\substack{a \in \omega^{-1}(\mathcal{D}^c),\\\Delta_a > 0}} \frac{16 \log(n)}{\Delta_a}\right) + 2 \max_{a\in \mathcal{A}} \Delta_a.
\end{equation}
\end{prop}

\begin{proof}[Proof of \Cref{prop:at-ucb_ecr}]
    The expected cumulative regret is given by,
    \begin{align}
        \mathcal{R}_{n}(\mathcal{C}) &= n'\cdot C + \mathbb{E}\left[\sum_{t=1}^n\left(\mu_{a_1} - Y_{a^{(t)}}\right)\right].
    \end{align}
    We focus on bounding the cumulative regret incurred by AT-UCB during the UCB phase of AT-UCB, which is represented by the second term. With this goal in mind, consider the event,
    \begin{equation*}
        G = \left\{\hat{\mu}_{\omega(a_1)} \geq \max_{a' \in \mathcal{A}'} \hat{\mu}_{a'} - \epsilon(\abs)\right\} 
        \cap 
        \left(
            \bigcap\limits_{a' \in \mathcal{D}}\left\{ \hat{\mu}_{a'} < \max_{\tilde{a}' \in \mathcal{A}'}\hat{\mu}_{\tilde{a}'} - \epsilon(\abs) \right\}
        \right) 
        =: G_1 \cap \left( \bigcap_{a' \in \mathcal{D}} G_{a'}\right).
    \end{equation*}
    Note that $G_{1}$ corresponds to the event wherein the empirical mean reward $\hat{\mu}_{\omega(a_1)}$ associated with $\omega(a_{1})$ (the intervention to which the optimal action $a_{1}$ in $\mathcal{B}$ maps), is \emph{larger} than the highest observed empirical mean reward $\hat{\mu}_{a'}$ minus the threshold hyperparameter $\epsilon(\abs)$. In other words, when $G_{1}$ occurs $\omega(a_{1})$ is not included in $\hat{\mathcal{D}}$ and thus action $a_{1}$ is available when UCB is ran in AT-UCB.

    Meanwhile, $G_{a'}$ corresponds to the event where the empirical mean reward $\hat{\mu}_{a'}$ of action $a'$ is \emph{smaller} than the highest observed empirical mean reward $\hat{\mu}_{\tilde{a}'}$ minus the threshold hyperparameter $\epsilon(\abs)$. In other words, when $G_{a'}$ occurs action $a'$ is included $\hat{\mathcal{D}}$, and thus any base action $a \in \omega^{-1}(a')$ mapping to $a{'}$ is not considered once UCB is ran. When the events $G_{a'}$ for all $a' \in \mathcal{D}$ simultaneously occur, $\mathcal{D} \subseteq \hat{\mathcal{D}}$. That is, all actions $a' \in \mathcal{D}$ are filtered out by AT-UCB and any corresponding base actions $a \in \omega^{-1}(\mathcal{D})$ are not available when UCB is ran.

    In summary, the event $G_{1}$ ensures that $\mu_{\omega(a_{1})}$ is not \emph{relatively underestimated} so that the optimal base arm $a_{1}$ is not filtered out and available when UCB is ran whilst $\bigcap_{a'\in \mathcal{D}}G_{a'}$ ensures that all abstract arms in $\mathcal{D}$ are not \emph{relatively overestimated} so that all base arms corresponding to an abstract arm $a' \in \mathcal{D}$ are filtered out before UCB is ran. Thus when the event $G$ occurs $\mathcal{D} \subseteq \hat{\mathcal{D}}$ and $\omega(a_{1}) \notin \hat{\mathcal{D}}$. That is, the optimal base arm $a_{1}$ is not filtered out and all base arms in $\omega^{-1}(\mathcal{D})$ are.
    
    
    Using the law of total probability, we may re-express cumulative regret incurred during the UCB phase of AT-UCB by conditioning on the event $G$:
    \begin{align}
        \mathbb{E}\left[\sum_{t=1}^n\left(\mu_{a_1} - Y_{a^{(t)}}\right)\right] &= \mathbb{E}\left[\sum_{t=1}^n\left(\mu_{a_1} - Y_{a^{(t)}}\right) \Bigg\vert G \right] \cdot \mathbb{P}(G) + \mathbb{E}\left[\sum_{t=1}^n\left(\mu_{a_1} -Y_{a^{(t)}}\right) \Bigg\vert G^c \right] \cdot \mathbb{P}(G^c)\\
        &\leq \mathbb{E}\left[\sum_{t=1}^n\left(\mu_{a_1} - Y_{a^{(t)}}\right) \Bigg\vert G \right] \cdot 1 + n \max_{a\in\mathcal{A}} \Delta_{a} \cdot \mathbb{P}(G^c),
    \end{align}
    where the inequality is justified by that fact that $\mathbb{P}(G) \leq 1$ and the fact that the expected regret on each time step is upper bounded by $\max_{a\in \mathcal{A}}\Delta_{a}$.
    

    It is clear from previous discussion that the cumulative regret conditioned on the event $G$ is upper-bounded by the cumulative regret of UCB applied to the arm set $\omega^{-1}\left(\mathcal{D}^c\right) = \cup_{a' \in \mathcal{D}^c} \omega^{-1}(a')$. Substituting the standard regret bound for UCB into the previous inequality yields  
    \begin{align}
        \label{eq:almost}
        \mathbb{E}\left[\sum_{t=1}^n\left(\mu_{a_1} - Y_{a^{(t)}}\right)\right] &\leq \left(3\sum_{a \in \omega^{-1}(\mathcal{D}^c)} \Delta_{a} + \sum_{\substack{a \in \omega^{-1}(\mathcal{D}^c),\\\Delta_a > 0}} \frac{16 \log(n)}{\Delta_a}\right) + n \max_{a\in\mathcal{A}} \Delta_{a} \cdot \mathbb{P}(G^c).
    \end{align}
    
    All that remains is to upper bound $\mathbb{P}(G^c)$. From the union bound it follows that:
    \begin{equation}
        \label{eq:union_bound}
        \mathbb{P}(G^c) \leq \mathbb{P}(G^c_1) + \sum_{a' \in \mathcal{D}}\mathbb{P}(G_{a'}^c).
    \end{equation}
    We will proceed by bounding $\mathbb{P}(G_{1}^{c})$ and $\mathbb{P}(G^{c}_{a'})$ for arbitrary $a' \in \mathcal{D}$. For $\mathbb{P}(G^{c}_{1})$, note that
    \begin{flalign}
        && \mathbb{P}(G^c_{1}) &= \mathbb{P}(\hat{\mu}_{\omega(a_1)} < \max_{a' \in \mathcal{A}'} \hat{\mu}_{a'} - \epsilon(\abs)) &&\\
        && &\leq \mathbb{P}\left(\hat{\mu}_{\omega(a_1)} \leq \max_{a' \in \mathcal{A}'} \hat{\mu}_{a'} - \epsilon(\abs)\right)&&\\
        && &\overset{(a)}{\leq} \sum_{a' \in \mathcal{A}'} \mathbb{P}\left(\hat{\mu}_{\omega(a_1)} \leq \hat{\mu}_{a'} - \epsilon(\abs)\right)&&\\
        && &= \sum_{a' \in \mathcal{A}'} \mathbb{P}\left(\epsilon(\abs) - \mu_{a'} + \mu_{\omega(a_1)} \leq \hat{\mu}_{a'} - \mu_{a'} + \mu_{\omega(a_1)} - \hat{\mu}_{\omega(a_1)}\right)&&\\
        && &= \sum_{a' \in \mathcal{A}'} \mathbb{P}\left(\epsilon(\abs) + \Delta'_{a'} - \Delta'_{\omega(a_1)} \leq \hat{\mu}_{a'} - \mu_{a'} + \mu_{\omega(a_1)} - \hat{\mu}_{\omega(a_1)}\right)&&\\\label{eq:bound_g1}
        && &\overset{(b)}{\leq} \sum_{a' \in \mathcal{A}'} \exp\left(- \frac{n'}{4k'}\left[ \epsilon(\abs) + \Delta'_{a'} - \Delta'_{\omega(a_1)} \right]^2\right),&&
    \end{flalign}
    where (a) follows from the union bound. Meanwhile, step (b) follows from the fact that $\hat{\mu}_{a'} - \mu_{a'}$ and $\mu_{\omega(a_1)} - \hat{\mu}_{\omega(a_1)}$ are 1-subgaussian, which in turn implies that $\hat{\mu}_{a'} - \mu_{a'} + \mu_{\omega(a_1)} - \hat{\mu}_{\omega(a_1)}$ is $\sqrt{2k'/n'}$-subgaussian. Note that in order to apply the definition of subgaussian variables we must first ensure that $\epsilon(\abs) + \Delta'_{a'} - \Delta'_{\omega(a_1)}$ is nonnegative. However this follows from the below chain of inequalities, justified by the nonnegativity of $\Delta'_{a}$ for all $a' \in \mathcal{A}'$ and \Cref{prop:bound_absoptbase}:
    \begin{equation}
        \epsilon(\abs) + \Delta_{a'}' \geq \epsilon(\abs) \geq \Delta_{\omega(a_1)}'.
    \end{equation}
    Next, we bound $\mathbb{P}(G_{a'}^c)$ as follows:
    \begin{flalign}
        && \mathbb{P}(G_{a'}^c) &= \mathbb{P}\left(\hat{\mu}_{a'} \geq \max_{\tilde{a}'\in\mathcal{A}'}\hat{\mu}_{\tilde{a}'} - \epsilon(\abs)\right)&&\\
        && &\overset{(a)}{\leq} \mathbb{P}\left(\hat{\mu}_{a'} \geq \hat{\mu}_{a_1'} - \epsilon(\abs)\right)&&\\
        && &\leq \mathbb{P}\left(\hat{\mu}_{a'} - \mu_{a'} + \mu_{a_1'} - \hat{\mu}_{a_1'} \geq \Delta'_{a'} - \epsilon(\abs)\right)&&\\\label{eq:bound_g2}
        && &\overset{(b)}{\leq} \exp\left(-\frac{n'}{4k'}\left[\Delta'_{a'} - \epsilon(\abs)\right]^2\right)
    \end{flalign}
    where (a) follows from the fact that $\hat{\mu}_{a_1'} \leq \max_{a' \in \mathcal{A}'} \hat{\mu}_{a'}$ as implied by $a_1' \in \mathcal{A}'$. Similar to before, (b) follows by appealing to the definition of subgaussian random variables.
    Note that the positivity of $\Delta'_{a'} - \epsilon(\abs)$ follows from the  definition $\mathcal{D}$ which implies $\Delta'_{a'} > \epsilon(\abs)$ for all $a' \in \mathcal{D}$ .

    Substituting \Cref{eq:bound_g1} and \Cref{eq:bound_g2} into \Cref{eq:union_bound} and applying \Cref{lem:bound_sum_exps} in conjunction with \Cref{ass:abs_hor_len} yields,
    \begin{flalign}
        \label{eq:subs_union}
        \mathbb{P}(G^c) &\leq  \sum_{a' \in \mathcal{A}'} \exp\left(-\frac{n'}{4k'}\left(\epsilon + \Delta_{a'}' - \Delta'_{\omega(a_1)}\right)^2\right) + \sum_{a' \in \mathcal{D}} \exp\left(-\frac{n'}{4k'}\left(\Delta'_{a'} - \epsilon\right)^2\right) \leq \frac{2}{n} 
    \end{flalign}
    Finally, substituting \Cref{eq:subs_union} into \Cref{eq:almost} yields the desired regret bound:
    \begin{equation}
        \mathbb{E}\left[\sum_{t=1}^n\left(\mu_{a_1} - Y_{a^{(t)}}\right)\right] \leq \left(3\sum_{a \in \omega^{-1}(\mathcal{D}^c)} \Delta_{a} + \sum_{\substack{a \in \omega^{-1}(\mathcal{D}^c),\\\Delta_a > 0}} \frac{16 \log(n)}{\Delta_a}\right) + 2 \max_{a\in \mathcal{A}} \Delta_a.
    \end{equation}
\end{proof}

\section{Further experimental details}\label{app:exp}

\subsection{Base model definition}\label{app:sirs_comms}

In our experiment, we consider as our base model $\mathcal{M}$ a discrete population model of an epidemic. In particular, we consider $k$ communities labelled by $c = 1, \ldots, k$, within which the number of individuals who are susceptible to the virus, infected by the virus, and recovered from the virus at time $t \in [0,T]$ are denoted, respectively, by $S_{c,t}$, $I_{c,t}$, and $R_{c,t}$. Each of the $k$ communities consists of an equal number of $N$ individuals, and the simulation progresses in time steps of size $\Delta t = 0.5$ until time $T = 50$. Between times $t$ and $t + \Delta t$, the number of individuals in community $c$ that transition from a state of susceptibility to a state of infection is modelled as
\begin{equation}
    N_{c,S\to I,t} \sim \text{Binomial}\left(S_{c,t}, l_{c,t}\left[1 - 
    \exp\left\{-\beta_{c} \Delta t \cdot \frac{\sum_{c' = 1}^k I_{c',t}}{k N}\right\}\right]\right),
\end{equation}
where $\beta_c > 0$ is a community-dependent parameter governing the ease with which individuals within that community become infected, and $l_{c,t} \in \{0,1\}$ assumes the value $0$ if a lockdown is in force at time $t$ and is $1$ otherwise. 
Secondly, the number of individuals transitioning from a state of infection to a recovered state is modelled as
\begin{equation}
    N_{c,I\to R,t} \sim \text{Binomial}\left(I_{c,t}, 1 - \exp\left\{-\gamma_{c} \Delta t\right\}\right),
\end{equation}
where $\gamma_c > 0$ is another community-dependent parameter. Finally, the number of individuals transitioning from a recovered state back to a state of susceptibility is modelled as
\begin{equation}
    N_{c,R\to S,t}\sim \text{Binomial}\left(R_{c,t}, 1 - \exp\left\{-\zeta_{c} \Delta t\right\}\right),
\end{equation}
where $\zeta_c > 0$ is a final community-dependent parameter. 

In our experiment, we take $k = 10$ communities/regions who are initialised to have $S_{c,0} = 0.9\cdot N$, $I_{c,0} = 0.1\cdot N$, $R_{c,0} = 0$. Interventions take the form of an action $a \in \{1,\ldots,k\}$ such that taking action $a$ sets $l_{a,t} = 0$ for $t \in [0.1 \cdot T, 0.5\cdot T]$. In our experiment, we further take $N=100$, and take the target variable to be the negative of the total number of individuals that are infected over time, i.e.:
\begin{equation}
    Y = - \sum_{c=1}^k \sum_{j=0}^{T / \Delta t} I_{c,j\cdot \Delta t}.
\end{equation}

\subsection{Approximate causal abstraction and abstracted model}\label{app:sirs_comms_abs}

To construct a high-level model $\mathcal{M}'$ and an approximate causal abstraction between $\mathcal{M}$ and $\mathcal{M}'$, we define a model in the same way as above using $k' = 4$ communities of $N' = 50$ individuals each, and group communities into composite communities via the $\tau$ map. The abstract model $\mathcal{M}'$ is also simulated with a larger time step size of $\Delta t' = 2$. Analogosuly to $\mathcal{M}$, we take the abstract target variable $Y'$ to be the negative of the total number of infections over time across all communities in the abstract model:
\begin{equation}
    Y' = - \sum_{c'=1}^{k'} \sum_{j=0}^{T/\Delta t'} I'_{c',j\Delta t'}.
\end{equation}

\subsection{Additional experimental details}\label{app:add_exp}

For both AT-UCB and UCB, we use $\delta = 0.1$ in our experiments.

Code for reproducing the experiments will be released upon acceptance.

All experiments were run on a Macbook Pro 2022 model with M2 chip, requiring $\sim 1$ hour of CPU time.

\section{Limitations of AT-UCB}
\label{app:lims}
As discussed in \Cref{sec:con}, AT-UCB has several limitations that we now discuss in more detail. Recall that \Cref{ass:abs_hor_len} provides a sufficient lower bound on $n'$, the number of exploratory actions taken in $\mathcal{B}^{'}$, which ensures the set $\mathcal{D}$ can be identified with high probability. Unfortunately the lower bound provided by \Cref{ass:abs_hor_len} depends on the optimality gap associated with each abstract arm $a' \in \mathcal{A'}$. As a result, setting $n'$ in practice requires deep knowledge of the abstract CMAB $\mathcal{B}'$, that is unlikely to be available a priori. 

Note that the AT-UCB algorithm bears resemblance to the traditional explore-then-commit (ETC) algorithm \citep{lattimore2020bandit}. Whilst ETC starts with an initial exploration phase comprised of uniform sampling before committing to the best empirical arm, AT-UCB performs uniform exploration at the abstract level before committing to a subset of of near-optimal arms. Similarly to AT-UCB, tuning the length of the exploration phase in ETC is challenging, with the optimal choice depending on the reward gaps of each arm. 

Observe that setting $n'$ in accordance with \Cref{ass:abs_hor_len} also requires knowledge of $\epsilon(\abs)$. Similarly, the regret bound provided by \Cref{prop:at-ucb_ecr} holds only when the hyperparameter $\varepsilon$ is set to $\epsilon(\abs)$. In other words, the theoretical guarantees presented rely on the learner knowing $\epsilon(\abs)$. In many scenarios, this requirement is not problematic. For example, the learner may be able to consult with a domain expert who can provide a trustworthy upper bound on $\epsilon(\abs)$. However, there are settings where it is unreasonable to expect the learner to know $\epsilon(\abs)$, such as when $\mathcal{M}$ and $\mathcal{M}'$ are hard-to-interpret black-box models. Note that $\epsilon(\abs)$ relies on the IC and RD errors, which both provide bounds on the divergence between reward distributions that hold uniformly over pairs of interventions. As a result, actively exploring both $\mathcal{B}$ and $\mathcal{B}'$ to learn $\epsilon(\abs)$ is costly, and developing CAMAB algorithms that do not rely on knowledge of $\epsilon(\abs)$ is non-trivial.

Finally, note that AT-UCB does not exploit any of the causal information encoded in the SCMs $\mathcal{M}$ and $\mathcal{M}'$. As discussed in \Cref{app:literatute_review}, many algorithms have been proposed for the CMAB setting that exploit partial knowledge of the underlying SCM to facilitate exploration. The fundamental ideas behind such algorithms could be used to augment the uniform exploration phase of AT-UCB when given appropriate partial knowledge of the abstract SCM $\mathcal{M}'$.

\end{document}